\RequirePackage{fix-cm}
\documentclass[smallextended]{svjour3}       % onecolumn (second format)
\smartqed  % flush right qed marks, e.g. at end of proof
\usepackage{graphicx}
\usepackage{lipsum}
\usepackage{amsfonts}
\usepackage{graphicx}
\usepackage{epstopdf}
\usepackage{algorithmic}
\usepackage{soul}
\ifpdf
  \DeclareGraphicsExtensions{.eps,.pdf,.png,.jpg}
\else
  \DeclareGraphicsExtensions{.eps}
\fi
\newcommand{\be}{\begin{equation}}
\newcommand{\ee}{\end{equation}}

\newcommand{\bma}{\begin{pmatrix}}
\newcommand{\ema}{\end{pmatrix}}
\newcommand{\Dc}{\mathcal{D}}

\newcommand{\Fc}{\mathcal{F}}
\newcommand{\Gc}{\mathcal{G}}

\newcommand{\Lc}{\mathcal{L}}

\newcommand{\Oc}{\mathcal{O}}
\newcommand{\Pc}{\mathcal{P}}

\newcommand{\Rc}{\mathcal{R}}

\newcommand{\Gal}{\tilde{ \Gamma} }
\newcommand{\Gd}{\mathcal{G}^{\dagger}}
\newcommand{\Gdex}{\mathcal{G}^{\dagger}_{\text{ext}}}
\newcommand{\Gpar}{\mathcal{G}_{\theta}}

\renewcommand{\O}{\Omega}

\renewcommand{\O}{\Omega}

\newcommand{\Hb}{\mathbb{H}}
\newcommand{\Ib}{\mathbb{I}}

\newcommand{\Nb}{\mathbb{N}}

\newcommand{\Rb}{\mathbb{R}}

\newcommand{\Beq}{\begin{equation}}
\newcommand{\Eeq}{\end{equation}}

\newcommand{\beq}{\begin{equation*}}
\newcommand{\eeq}{\end{equation*}}
\newcommand{\bal}{\begin{align}}
\newcommand{\eal}{\end{align}}
\usepackage{graphicx, epsfig}
\graphicspath{{Figs_1/}}
\usepackage{bbm}
\usepackage{subcaption}
\usepackage{verbatim,url}
\usepackage{amssymb,amsmath}
\usepackage{dsfont}
\newcommand{\required}[1]{\section*{\hfil \sharp1\hfil}}           \usepackage{tikz}       
\renewcommand{\refname}{\hfil References Cited\hfil}   

\usepackage{hyperref}
\usepackage{upref}

\begin{document}

\title{A DeepONet for inverting the Neumann-to-Dirichlet operator in Electrical Impedance Tomography: An approximation theoretic perspective and numerical results%\thanks{Grants or other notes
}
\subtitle{}


\author{Anuj Abhishek         \and
        Thilo Strauss %etc.
}


\institute{Anuj Abhishek \at
             Department of Mathematics, Applied Mathematics and Statistics,\\ Case Western Reserve University \\
              \email{anuj.abhishek@case.edu}           %  \\
           \and
           Thilo Strauss \at
           School of AI and Advanced Computing,\\
             Xi'an Jiaotong Liverpool University\\
             \email{Thilo.Strauss@xjtlu.edu.cn}
}

\date{Received: date / Accepted: date}

\maketitle

\begin{abstract}
In this work, we consider the non-invasive medical imaging modality of 
Electrical Impedance Tomography (EIT), where the goal is to recover the 
conductivity in a medium from boundary current-to-voltage measurements, i.e., 
the Neumann-to-Dirichlet (N--t--D) operator. We formulate this inverse problem 
as an operator-learning task, where the aim is to approximate the implicitly 
defined map from N--t--D operators to admissible conductivities. To this end, 
we employ a Deep Operator Network (DeepONet) architecture, thereby extending 
operator learning beyond the classical function-to-function setting to the more 
challenging operator-to-function regime. We establish a universal approximation 
theorem that guarantees that such operator-to-function maps can be approximated 
arbitrarily well by DeepONets. Furthermore, we provide a computational 
implementation of our approach and compare it against the iteratively 
regularized Gauss--Newton (IRGN) method. Our results show that the proposed 
framework yields accurate and robust reconstructions, outperforms the baseline, 
and demonstrates strong generalization. To our knowledge, this is the first 
work that combines rigorous approximation-theoretic guarantees with 
DeepONet-based inversion for EIT, thereby opening a principled and interpretable 
pathway for use of DeepONets in such inverse problems.

\end{abstract}

\section{Introduction}
In recent years tremendous progress has been made in developing accurate, non-invasive medical imaging techniques for disease detection and diagnosis and to make it suitable for practical deployment. One such medical imaging technology is Electrical Impedance Tomography (EIT). EIT is a functional imaging technique that seeks to reconstruct the spatial conductivity within a domain by applying currents and measuring voltages across electrodes attached to the boundary of the domain of interest. EIT has been shown to be clinically useful in monitoring lung function and breast cancer, see e.g. \cite{chere_02,frer_02}. It is now becoming an increasingly promising neuro-imaging method, e.g. in early detection of cerebral edema and monitoring the effect of mannitol dehydration, \cite{yang_19}. It is also being used in epilepsy imaging to detect seizure onset zones, \cite{wro_18}.
Moreover, a bedside EIT stroke monitoring system has been developed for clinical practice and there is growing consensus among practitioners for utilizing EIT in multi-modal imaging for diagnosing and classifying strokes, see e.g.\cite{Agnelli_2020,rand_21}. In spite of the fact that EIT has lower spatial
resolution than some other tomographic techniques such as computed tomography and MRI, it has relatively
fast data acquisition times, relatively inexpensive hardware and does not use ionizing radiations, making
it a versatile and safer alternative to many other medical imaging modalities, see \cite{yang_22} and the references therein. 
\par Recall that a typical EIT experiment consists of injecting a weak excitation current through boundary electrodes surrounding the region of interest and imaging the distribution of electrical parameters in the region of interest by measuring voltage signals on such electrodes. Thus the measured data arises as a current-to-voltage map, which is physically related to the internal conductivity in the region of interest. Different kinds of tissues can be differentiated based on their different conductivity values, which may also be altered in diseased states. As such, the inverse problem consists of detecting physiological and pathological changes in the conductivity profile based on measurements of the current-to-voltage map across the boundary electrodes. More recently, there is a growing interest among theoreticians and practitioners alike in exploiting various machine-learning based algorithms and neural-network architectures to solve the EIT inverse problem. Against this background, in this work our main contributions are as follows:
\begin{enumerate}
\item  We successfully utilize the DeepONet-based approach for learning the map from Neumann-to-Dirichlet operators to conductivity functions in Electrical Impedance Tomography, extending DeepONet-based operator learning beyond learning a map between function spaces.
 \item We establish universal approximation guarantees for this operator-to-function setting, providing theoretical justification for why such architectures are effective and interpretable in critical applications like medical imaging.
\item  We show that the DeepONet architecture, closely related to implicit neural representations in computer vision, achieves state-of-the-art performance compared to classical baselines, while also exhibiting robustness under noise.
\end{enumerate}
\par For the benefit of the readers, in the following subsections we will first give a self-contained mathematical description of the inverse problem at hand and survey some related literature. In the subsequent sections, we will describe our main theoretical results, give details about the proposed architecture and support our findings with numerical experiments.
%We begin now by summarizing our main contributions.

%subsection{Contribution}
%Our contributions can be summarized in the following points:
%\begin{itemize}
%    \item In this work, we propose a DeepONet architecture for solving the EIT inverse problem and provide a rigorous mathematical justification for using such an architecture.
%    \item Various approximation-theoretic results for operator networks which seek to learn operators between infinite-dimensional function spaces have appeared in the recent past. We extend those results in order to make them suitable for the EIT inverse problem by proving a Universal Approximation Theorem type result for the case when the input space lies in the class of Hilbert-Schmidt operators (related to the current-to-voltage map) and the output space is a function space (corresponding to the conductivity function).
%    \item We show that our proposed method achieves good reconstructions and show that it outperforms the IRGN (iterative regularized Gauss-Newton) baseline.
%\end{itemize}
\subsection{\textbf{Mathematical preliminaries and background of the EIT problem}}
 We will begin by introducing the  notation used in this work: Let $\Omega\subset \Rb^d$ for $d\geq 2$ be some simply connected domain with smooth boundary and $\Omega^{\prime}\subset \Omega$ be a compact set. We denote the set of essentially (i.e. almost surely) bounded functions by $L^{\infty}(\Omega)$ and the subset of essentially positive functions by $L^{\infty}_{+}(\Omega).$ In general $H^s(\Omega)$ will be used to refer to the space of $L^2-$ based Sobolev space of functions on some smooth bounded domain $\Omega$. Additionally, $H^1_{\diamond}(\Omega)$ and $L^2_{\diamond}(\Omega)$ will be used to denote the spaces of $H^1$ and $L^2$ functions with vanishing integral mean on the boundary $\partial \Omega$, i.e. $h\in H^1_{\diamond}(\Omega):=\{h\in H^1(\Omega): \int_{\partial \Omega} h dS=0 \}$ and similarly,  $g\in L^2_{\diamond}(\Omega):=\{g\in L^2(\Omega): \int_{\partial \Omega} g dS=0 \}$. We will also denote the space of linear operators from some vector space $X$ to $Y$ by {{}{$\Lc(X,Y)$}}. Mathematically, the problem can be formulated as recovery of the conductivity function $\gamma(x)$ in some simply connected domain $\Omega$ based on electric potential and current measurements made entirely on the boundary $\partial \Omega$ (or, some subset of it $\Sigma\subset \partial \Omega$). We will first describe an abstract model which has been studied by many authors since it was first described in \cite{cald_80}. To formulate the problem  consider the following elliptic partial differential equation with associated Neumann boundary-data which serves as a model for EIT:
\begin{align}
 \nabla \cdot (\gamma \nabla u)&=0 \quad \mbox{ in } \Omega\label{eq:1}; \\
    \begin{split}
    {{}{\gamma\, }}\partial_{\nu}u|_{\partial \Omega}&= g \mbox{ on } \partial \Omega \quad(\text{or more generally } \Sigma)\\
    &=0 \mbox{ else.}
    \end{split}
\end{align} Here $\gamma\in L^{\infty}_{+}(\Omega)$ is the internal conductivity in the medium $\Omega$ such that $\gamma=1 \mbox { in } \Omega\setminus \Omega^{\prime}$, and $u(x)$ is the electric potential in the interior. Loosely speaking, we will refer to $\gamma\equiv 1$ as the background conductivity. The Neumann boundary-data models the applied current. In the abstract continuum model (CM), the corresponding boundary measurements of the voltage (Dirichlet data) are given by the  Neumann-to-Dirichlet (N-t-D) operator defined on  the boundary $\partial \Omega$ ( or more generally on a subset $\Sigma\subset\partial \Omega$):
\begin{align}
    \Lambda_{\gamma}:L^2_{\diamond}(\partial \Omega)\to L^2_{\diamond}(\partial \Omega), \quad g\mapsto u|_{\partial \Omega}.
\end{align}To emphasize, $u|_{\partial \Omega}$ refer to the voltage measurements on the the boundary $\partial \Omega$. Additionally, here and below we will use $\Lambda_1$ to denote the background N-t-D operator corresponding to the background conductivity. 
 We would like to mention here, that in the abstract mathematical formulation, typically the Dirichlet-to-Neumann (D-t-N) map is considered instead of the N-t-D map as is considered in this work. However, in practical settings such as in EIT, the data that is generated is in the form of current-to-voltage maps which is best captured by the N-t-D maps. A more practical formulation of the EIT problem is given by the Complete Electrode Model (CEM), \cite{ch_92}. As before we let $\gamma\in L^{\infty}_{+}(\Omega)$. Let there be $M$ open, connected and mutually disjoint electrodes, $E_m\subset \partial \Omega$, $m=\{1,\cdots,M\}$ having the same contact impedance $z>0$. Assuming that the injected currents $I_m$ on the electrodes satisfy the Kirchoff condition, i.e. $\Sigma_{m=1}^M I_m=0$, the resulting electric potential $(u,U)\in H^1(\Omega)\times \Rb^M$ satisfies \eqref{eq:1} along with the following conditions on the boundary:
\begin{align}
    \gamma \partial_{\nu}u&=0; \quad \text{on } \partial \Omega\setminus \bigcup E_m\\
    \int_{E_m}\gamma \partial_{\nu}u dS&= I_m \quad \text{on } E_m, m=1,\cdots, M\\
    u+z\gamma\partial_{\nu}u&= U_m \quad\text{on }E_m, m=\{1,\cdots,M\}
\end{align}
where the vector $U=(U_1,\cdots,U_M)$ represents voltage measurement on the electrodes satisfying a ground condition $\Sigma_{m=1}^M U_m=0$. We can thus define the $M-$electrode current to voltage operator,
$R_M(\gamma):\Rb^M_{\diamond}\to \Rb^M_{\diamond}: I:=(I_1,\cdots,I_M)\mapsto (U_1,\cdots, U_M)=U$ and $U=R_M(\gamma)I$. For later use, we also define the `shifted' CEM operator, $\tilde{R}_{M}(\gamma):={R}_{M}(\gamma)-{R}_{M}(1)$ where ${R}_{M}(1)$ is the background CEM current-to-voltage operator. Since the background operator {{}{${R}_{M}(1)$}} is known, $\tilde{R}_M(\gamma)$ is completely determined by $R_M(\gamma)$ and vice-versa. Before we define the inverse problem at hand, we will first define the set of admissible conductivities considered in this work.  We recall the following definition from \cite[Definition 2.2]{harr_19}. 
\begin{definition}
{{}{Note that for bounded smooth domains, $L^{\infty}{ (\Omega)}\cap L^2(\Omega)=L^{\infty}(\Omega)$.}} A set $\Fc\subset L^{\infty}{ (\Omega)} $ is called a finite dimensional subset of piecewise analytic functions if its linear span 
$$\mathrm{span }\ \Fc=\bigg\{ \sum\limits_{j=1}^{k}\lambda_j f_j:k \in \Nb,\lambda_j\in \Rb,f_j\in \Fc\bigg\}$$ contains only piecewise analytic function and $\mathrm{dim }\ (\mathrm{span }\ \Fc)<\infty.$
\end{definition}
Given such a finite dimensional subset $\Fc$ and two positive numbers $m,M$ we denote the space of \textit{admissible conductivities} $\Gamma$ as:
\begin{align}
 \Gamma&=\{\gamma\in \Fc :M>\gamma (x) \geq m>0, \gamma(x)=1 \text{ in }\Omega\setminus\Omega^{\prime}   \}  
\end{align}
As $\Omega$ is bounded, any such admissible conductivity is clearly contained in $L^2(\Omega). $ Furthermore, $\Gamma$ is closed by virtue of being a known finite dimensional subspace of $L^{\infty}(\Omega)$. Such conductivities have been considered in the literature before, see e.g. \cite{harr_19,al_19}. We note that the class of conductivities while restrictive, has been vastly used in the literature before and allows for the development of rigorous error bounds below which crucially use the Lipschitz stability estimates developed for the inverse problem at hand, see \cite{harr_19}.  Now we are ready to state the inverse problem in CM: \textit{given all possible pairs of associated functions $(g,u|_{\Sigma})$, we want to find $\gamma\in \Gamma$}.

\subsection{\textbf{Related Research}}

An EIT experiment applies the electrical current (Neumann data) on $\partial\Omega$ to measure the electrical potential differences on $\partial\Omega$. By doing so, the so-called Neumann-to-Dirichlet (NtD) operator {{}{$\Lambda_{\gamma}$}} is obtained. The inverse problem is reconstructing the unknown conductivity {{}{$\gamma$}} from a set of EIT experiments \cite{borcea2002electrical,cheney1999electrical,hanke2003recent,lionheart2004eit}. Early important results by Sylvester and Uhlmann \cite{SU}, and by Nachmann \cite{nach} established that unique recovery of the conductivity from the boundary data in the form of D-t-N map is possible in the (abstract) Calder\'{o}n problem set-up. Stability estimates for the Calder\'{o}n problem were given in \cite{ales,NS10}. Optimality of such stability estimates, at least up to the order of the exponents, was proved in \cite{Mand_01}. \cite{Salo_lec,U_sur} and many references therein contain an exhaustive survey of results in this area.

%The many applications of EIT include monitoring soil \cite{daily1992electrical}, medical imaging \cite{holder2004electrical, bayford2006bioimpedance}, or crack detection \cite{hou2009electrical}.
EIT is a highly nonlinear and strongly ill-posed problem. Therefore, in the classical literature, EIT requires a regularization for a good reconstruction \cite{jin2012reconstruction,jin2012analysis}. In the literature, there are abundant approaches for solving the EIT inverse problem including the factorization method \cite{kirsch2008factorization}, d-bar method \cite{Isaacson:2004}, or variational methods for least-square fitting \cite{jin2012reconstruction,jin2012analysis,ahmad2019comparison,khan20051d} as well as, on statistical methods \cite{NiAb19,ahmad2019comparison,kaipio2004posterior,strauss2015statistical}.

There is also some recent research on deep learning for solving the EIT inverse problem. They include using Physics inspired neural networks \cite{pokkunuru2023improved}, different kinds of convolutional neural networks  \cite{hamilton2018deep,fan2020solving,hamilton2019beltrami}, and shape reconstruction \cite{strauss2023implicit}. On the other hand, an architecture called Deep-Operator-Networks (DeepONet, or DON) was proposed in \cite{Lu_21} for learning implicitly defined operator maps between infinite-dimensional function spaces. This architecture was a massive generalization of an earlier known (shallow) operator-network architecture proposed in \cite{chen_95}, see also \cite{chen_93,chen_95_2}. Subsequently, variations of DeepONet (more generally, neural-operators) have been used for solving important physical problems and substantial efforts have been made to understand the theoretical underpinnings of why this approach seems to be particularly efficient, see e.g. \cite{Zhu2023,He2023,shih2024,Peyvan_2024,karn_23a,karn_24a,li2023,Raonic23,goswami2022,Bart_23,Li24,kovachki2024,deHoop23,kovachki_23,li2021,kov21} and the many references therein. The main idea in operator-learning is that it seeks to directly approximate the operator map between infinite-dimensional spaces and in many cases can be shown to approximate any continuous operator to desired accuracy. Hereby, we do not have the disadvantages of the somewhat limited resolution of traditional convolutional neural-network based methods such as \cite{hamilton2018deep,fan2020solving,hamilton2019beltrami}. \textit{While most of the work in the literature has so far appeared in approximating a map between two function spaces, in EIT, we are interesting in learning the map between a space of (N-t-D) operators to a space of (conductivity) functions.} Furthermore, we do not only provide strong experimental results for the DeepONet implementation for EIT but also provide approximation-theoretic guarantees as a theoretical justification for this approach. 

\par In the computer-vision community, a similar DeepONet based architecture was introduced as implicit-neural-networks or, INNs. The idea here is to replace conventional discretized signal representation with implicit representations of 3D objects via neural networks. In discrete representations images are represented as discrete grids of pixels, audio signals are represented as discrete samples of amplitudes, and 3D shapes are usually parameterized as grids of voxels, point clouds, or meshes. In contrast, Implicit Neural Representations seek to represent a signal as a continuous function that maps the domain of the signal (i.e., a coordinate, such as a pixel coordinate for an image) to the value of the continuous signal at that coordinate (for an image, an R,G,B color). Some early work in this field includes single-image 3D reconstructions \cite{mescheder2019occupancy,hu2017deep,Sitzmann20}, representing texture on 3D objects \cite{oechsle2019texture}, representing surface light fields in 3D \cite{oechsle2020learning}, or 3D reconstructions from many images \cite{mildenhall2020nerf}. \textit{Both DeepONets and INNs seek to learn implicitly defined continuous maps between the input and output spaces and thus share many similarities.} Motivated by these works, one of the authors of this article used such network architectures for shape reconstruction in the EIT problem \cite{strauss2023implicit}.

\section{The Learning Problem for EIT}\label{sec:5}

 Our goal in this article is to propose a DeepONet based neural-network architecture for learning to invert Neumann-to-Dirichlet operator relevant to the problem of EIT. To formally set up the learning problem, let us consider the shifted abstract N-t-D operator $\tilde{\Lambda}_{\gamma}={\Lambda}_{\gamma}-\Lambda_{1}$, where $\Lambda_{1}$ is the background N-t-D operator. Analogous to the CEM case, $\tilde{\Lambda}_{\gamma}$ is completely determined by $\Lambda_{\gamma}$ and vice-versa. Now we list some mapping properties of the operator $\tilde{\Lambda}_{\gamma}$ which will be crucial in our analysis. First of all, note that $\tilde{\Lambda}_{\gamma}$ is a  smoothing operator, see e.g. \cite[Theorem A.3]{hyvonen_09} and \cite[(2.5)]{garde_21}. In particular the operator norm of $\tilde{\Lambda}_{\gamma}$ is bounded, i.e., $\lvert\lvert \tilde{\Lambda}_{\gamma}\rvert\rvert_{\Lc( H^s\to H^t)}\leq C$ for all $s,t\in \Rb$. {}{In fact, following \cite{NiAb19,garde_21}, it can be shown that the operator $\tilde{\Lambda}_{\gamma}$ belongs to the class of Hilbert-Schmidt operators.} For readers' convenience we recall the definition of a Hilbert-Schmidt operator from \cite{NiAb19} here. Consider an orthonormal basis of $L^2(\partial \Omega)$ given by $
\{ \phi_k := \phi_k^{(0)} \}_{k=0}^{\infty}$ consisting of real-valued eigenfunctions of the Laplace–Beltrami operator on the compact manifold \( \partial \Omega \). 
By removing the constant function \( \phi_0 \)  from the above collection, we obtain bases for
$L^2(\partial \Omega)/\mathbb{C}$. {}{Indeed, $L^2_\diamond(\partial\Omega)$ is isomorphic to $L^2(\partial\Omega)/\mathbb{C}$ and we adopt the $L^2_\diamond$ notation to emphasize the zero-mean property. The space  $H^r_\diamond(\partial\Omega)$ is similarly interpreted. Also, we make the identification $L^2_\diamond(\partial\Omega)$ with $H^0_{\diamond}(\partial\Omega).$} Now, by an appropriately rescaling these basis functions we can also get orthonormal bases 
$\{ \phi_k^{(r)}  \}_{k=1}^{\infty} $
of all \( H^r(\partial \Omega)/\mathbb{C} \) and \( H^r_\diamond(\partial \Omega) \) spaces, \( r \in \mathbb{R} \). For \( j, k \in \mathbb{N} \), denote the tensor product operator, \( b_{jk}^{(r)} : H^r(\partial \Omega) \to L^2(\partial \Omega) \) 
\[
b_{jk}^{(r)}(f) = \phi_j^{(r)} \otimes \phi_k^{(0)} (f) := \langle f, \phi_j^{(r)} \rangle_{H^r(\partial D)} \phi_k^{(0)}, \quad f \in H^r(\partial \Omega),
\]
and define the space of linear operators
\[
\mathbb{H}_r := \left\{ T : H^r(\partial \Omega) \to L^2(\partial \Omega), \quad T = \sum_{j,k=1}^\infty t_{jk} b_{jk}^{(r)} : t_{jk} \in \mathbb{R}, \sum_{j,k=1}^\infty t_{jk}^2 < \infty \right\}.
\]
The elements of \( \mathbb{H}_r \) are the ‘Hilbert–Schmidt’ operators between \( H^r(\partial \Omega) \) and \( L^2(\partial \Omega) \). Furthermore, \( \mathbb{H}_r \) is itself a Hilbert space with inner product
\[
\langle S, T \rangle_{\mathbb{H}_r} := \sum_{j,k=1}^\infty s_{jk} t_{jk} \equiv \sum_{j,k=1}^\infty \langle S \phi_j^{(r)}, \phi_k^{(0)} \rangle_{L^2(\partial D)} \langle T \phi_j^{(r)}, \phi_k^{(0)} \rangle_{L^2(\partial D)}.
\] From \cite [Lemma 18]{NiAb19}, we know that $\tilde{\Lambda}_{\gamma}\in \Hb_r$ for any $r$. Consider the set $D_{\Lambda}=\{\tilde{\Lambda}_{\gamma}:\gamma\in \Gamma \}$ of N-t-D operators corresponding to admissible conductivities. We show in Lemma \ref{lem:1} that $D_{\Lambda}$ is a compact set of $\Hb_r$.

\par Formally, the problem is to `learn' the map $G^{\text{true}}:D_{\Lambda}\to\Gamma$. More particularly, we want to design a deep neural network (DeepONet) to approximate the map, $G^{\text{true}}$. For learning operators between infinite dimensional spaces, DeepONets were introduced in \cite{LM_22,Lu_21}. These generalized the (shallow-net) architecture proposed in \cite{chen_95_2,chen_95}. Theoretical guarantees for DeepONets were given in \cite{LM_22} where the neural network tries to learn operators between two Hilbert spaces. This theory was further generalized in \cite{kov21,kov_th} and an architecture called Fourier Neural Operators was introduced for learning operators between Banach spaces having a certain approximation property. Fairly recently, DeepONets and FNOs have been synthesized as Neural Inverse Operators and they have been used effectively in solving several inverse problems, \cite{mish_23}. 
\par {{} {Before moving on, we would like to emphasize an important distinction of the learning problem formulated in this work with the existing classical DeepONet setting. Although the (shifted) Neumann-to-Dirichlet operators considered in this work can be embedded in the Hilbert space $\mathbb{H}_r$ of Hilbert-Schmidt operators, the learning problem cannot be treated as a direct application of the standard function-to-function DeepONet framework. Classical DeepONet theory assumes that the input is a function that can be accessed through pointwise sensor evaluations. In contrast, in Electrical Impedance Tomography the input object is itself a linear operator, and the available measurements correspond to finite collections of electrode responses obtained by applying this operator to boundary current patterns, rather than pointwise samples of a function. As a result, connecting the physical measurement model with the operator-learning architecture requires additional structure. In particular, we identify an encoder-decoder construction that maps the continuum Neumann-to-Dirichlet operator to discrete electrode measurements and subsequently lifts these measurements back into an operator representation in $\mathbb{H}_r$. Establishing the validity of this representation relies on convergence results from EIT theory showing that the discrete CEM operator approximates the continuum N-t-D operator as the number of electrodes increases. These ingredients are not present in the standard DeepONet setting and are essential for formulating the inverse problem as learning an operator-to-function map.}}

With this background in mind, and to set the stage for formulating our learning problem as learning an operator between two Hilbert Spaces, let us introduce a  (regular)`link function' {$\Phi:\mathbb{R}\to [m,\infty)$} such that any conductivity $\gamma\in \Gamma$ is given by $\gamma =\Phi\circ f$ for some $f\in L^{\infty}(\Omega)\cap L^2{(\Omega)}:=\tilde{\Gamma}$. 
 We recall the following example of such a link function from \cite [Example 8]{NGW_20}. {{}{Consider the function \(\phi : \mathbb{R} \to (0,\infty)\) defined as:
$\phi(x) = e^{x}\mathbf{1}_{\{x<0\}} + (1+x)\mathbf{1}_{\{x \ge 0\}} .
$
where $\mathbf{1}_{A}$ denotes the characteristic function of a set $A$.
Let \(\psi : \mathbb{R} \to [0,\infty)\) be a smooth, compactly supported function satisfying
$\int_{\mathbb{R}} \psi(y)\,dy = 1,
$
and consider the convolution
$(\phi * \psi)(x) = \int_{\mathbb{R}} \phi(x-y)\psi(y) dy.
$
Then, for any \(K_{\min} \in \mathbb{R}\):
\[
\Phi : \mathbb{R} \to (K_{\min},\infty), \qquad 
\Phi(x) = K_{\min} + \frac{1-K_{\min}}{(\psi * \phi)(0)}\,(\psi * \phi)(x),
\]
is a regular link function with range \((K_{\min},\infty)\). Furthermore, we recall the following properties of such a link function from \cite[Lemma 19]{NGW_20} and \cite[eq. 16-19]{NiAb19} that will be needed later in our proofs: \begin{equation}
\|\Phi \circ f - \Phi \circ f_0\|_{\infty}
\leq
C \|f - f_0\|_{\infty},
\end{equation}
\begin{equation}
\|\Phi^{-1} \circ \gamma - \Phi^{-1} \circ \gamma_0\|_{\infty}
\leq
c \|\gamma - \gamma_0\|_{\infty},
\end{equation}
\begin{equation}
\|\Phi \circ f\|_{H^s(\Omega)}
\leq
C' \bigl(1 + \|f\|_{H^s(\Omega)}^s\bigr),
\end{equation}
and
\begin{equation}
\|\Phi^{-1} \circ \gamma_0\|_{H^s(\Omega)}
\leq
c' \bigl(1 + \|\gamma_0\|_{H^s(\Omega)}^s\bigr).
\end{equation}}} Now we will modify our operator learning problem. For this we first define a map, $\Gd:D_{\Lambda}\to \tilde{\Gamma}$ where $\Gd=\Phi^{-1}\circ G^{\text{true}}$. Thus knowing $\Gd$ completely determines $G^{\text{true}}$. Next we extend this map: Let $\mathcal{G}^{\dagger}:D_{\Lambda}\to L^{\infty}(\Omega)\cap L^2(\Omega)$ be the map as above, then there exists an extension $\mathcal{G}^{\dagger}_{\text{ext}}:\Hb_r\to L^{\infty}(\Omega)\cap L^2(\Omega)$ such that $\mathcal{G}^{\dagger}_{\text{ext}}|_{D_{\Lambda}}=\mathcal{G}^{\dagger}$. The extension $\Gdex$ exists by Lemma \ref{th:dug}. The map $\Gdex$ will be approximated by a \textit{deep neural operator} called DeepONet. Essentially, a neural  operator is a parametric map between the input space (e.g. space of N-t-D operators) and the output space (e.g. space of admissible conductivities) that can be interpreted as a composition of three maps, an \textit{encoder}, an \textit{approximator} and a \textit{reconstructor}, see \cite{LM_22} and Fig.\ \ref{Fig:a}. As such, an upper bound on the error in approximating the true operator by a neural operator can be split into upper bounds on the encoder error, the approximator error and the reconstructor error correspondingly. Now we recall the components of a DeepONet as proposed in \cite{LM_22,Lu_21} and adapt the definitions as applicable to the case of EIT. To that end, we define the following operators while keeping the terminology the same as in \cite{LM_22}:

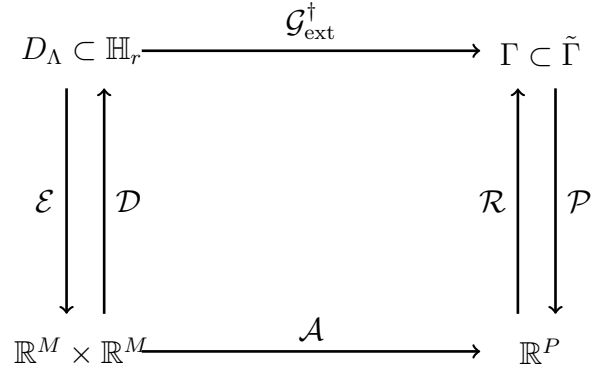
\begin{figure}
\begin{center} 
\begin{tikzpicture}
    % Define line width for bold arrows
    \tikzset{boldarrow/.style={->, line width=1.0pt}}

    % Parallel to y-axis arrows
    \draw[boldarrow] (0.0,3) to node[left] {$\mathcal{E}$} (0.0,0);
    \draw[boldarrow] (0.5,0) to node[right] {$\Dc$} (0.5,3);
    
    \draw[boldarrow] (6.0,0) to node[left] {$\Rc$} (6.0,3);
    \draw[boldarrow] (6.5,3) to node[right] {$\Pc$} (6.5,0);
    
    % Parallel to x-axis arrows
    \draw[boldarrow] (1.0,-0.5) to node[above] {$\mathcal{A}$} (5.5,-0.5);
    \draw[boldarrow] (1.0,3.5) to node[above] {$\mathcal{G}^{\dagger}_{\text{ext}}$} (5.5,3.5);
    
    % Labels 
    \draw node at (0.2,-0.5) {$\mathbb{R}^{M}\times \mathbb{R}^M$};
    \draw node at (0.2,3.5) {$D_{\Lambda}\subset\mathbb{H}_r$};
    \draw node at (6.3,-0.5) {$\mathbb{R}^{P}$};
    \draw node at (6.3,3.5) {$\Gamma\subset\tilde{\Gamma}$};
\end{tikzpicture}
\end{center}
\caption{The true map $\mathcal{G}^{\dagger}_{\text{ext}}$ is approximated by a composition of three maps, encoder $\mathcal{E}$, approximator $\mathcal{A}$ and reconstructor $\mathcal{R}$. The resultant error in the approximation thus comprises of encoder, approximator, and reconstructor errors. }
\label{Fig:a}
\end{figure}
 
 \begin{itemize}
     \item \textit{Encoder}: Given a set of $M$ electrodes covering the boundary of the domain, $E_m\subset \partial \Omega$, $m=\{1,\cdots,M\}$ we define the encoder:
     \begin{align}
       \mathcal{E}_M:=\mathcal{E}:D_{\Lambda}\subset\Hb_{r}\to \Rb^M\times \Rb^M;\quad \tilde{\Lambda}_{\gamma}\mapsto \tilde{R}_{M}(\gamma)
     \end{align}
     The encoder can be thought of as `sensor evaluations' in the spirit of \cite{LM_22}. Whenever it is clear from context, we will denote $\mathcal{E}_M$ by $\mathcal{E}$. It follows from Lemmas \ref{lem:7.3} and \ref{lem:7.4} that the map $\mathcal{E}$ is continuous and from Lemma \ref{th:dug} can be extended continuously on $\Hb_r$. With a slight abuse of notation, we call this extended map $\mathcal{E}$ too.
    \item\textit{Decoder:} While the encoder is a projection of an infinite-dimensional object (i.e. $\tilde{\Lambda}_{\gamma}\in \Hb_r$) into a finite-dimensional space, a decoder lifts back a finite dimensional representation, $R_{M}(\gamma)$ into its corresponding map in the space of Hilbert-Schmidt operators. This lifting is done using the construction given in \cite[section 6]{garde_21}  and results in an operator ${\Lambda}^M_{\gamma}$, see \cite[eqn. 6.8]{garde_21}. Thus, we have:
    \begin{align}
       \Dc_M:= \Dc:\Rb^M\times \Rb^M\to \Hb_{r}; \quad \tilde{R}_{M}(\gamma)\mapsto \tilde{\Lambda}^M_{\gamma}.
    \end{align}
    Similar to above, whenever it is clear from context, we will use the shorthand notation $\Dc$ instead of $\Dc_M$. Note that, $\Dc\circ\mathcal{E}(\tilde{\Lambda}_{\gamma})=\tilde{\Lambda}^M_{\gamma}.$ In particular, we use here as the decoder the construction outlined in \cite[Section 3]{hyv_08}. %Indeed, given any Neumann data $f\in L^2_{\diamond}(\partial\Omega)$, $\tilde\Lambda^M_{\gamma}f:= \tilde{R}_M \tilde{P}^{E}_M {f}$ where the map $\tilde{P}^{E}_M$ is a (auxiliary) projection operator defined in accordance with \cite[eq. (5.1)]{hyv_08}. 
    {{}{For the sake of completeness, we briefly outline this construction from the article \cite{hyv_08}. Indeed given any Neumann data $f \in L^2_\diamond(\partial\Omega)$ we define 
$\tilde{\Lambda}_\gamma^M f := \tilde{R}_M(\gamma)\tilde{P}_E^M f$, 
where the operator $\tilde{P}_E^M$ is an auxiliary projection operator introduced in \cite[eq. (5.1)]{hyv_08}. This operator $\tilde{P}_E^M$ maps a boundary current $f \in L^2(\partial\Omega)$ to an electrode current pattern in the following manner.
First let $P_E^M$ denote the projection that replaces the boundary current by its average value on each electrode. 
More precisely $(P_E^M f)(s) = \frac{1}{|E_j|}\int_{E_j} f dS$ for $s \in E_j$ and zero elsewhere. 
Since $P_E^M f$ does not in general belong to the space of admissible electrode currents the modified projection 
$\tilde{P}_E^M f = P_E^M f - Kf$ is introduced where 
$Kf = \frac{\sum_j (P_E^M f)_j |E_j|}{\sum_j |E_j|}$. 
The constant $Kf$ is chosen so that the resulting current pattern satisfies the conservation of charge condition in the complete electrode model. }}
    \item \textit{Approximator}: The approximator is a deep (feedforward) neural network defined as:
    \begin{align}
        \mathcal{A}: \Rb^M\times \Rb^M \to \Rb^{P}; \quad \tilde{R}_{M}(\gamma)\mapsto \{\mathcal{A}_k\}_{k=1}^P
    \end{align}
    The composition of encoder and approximator, i.e. $\beta:=\mathcal{A}\circ\mathcal{E}:\Hb^r\to \Rb^P$ will be referred to as the \textit{branch net}.
    \item \textit{Reconstructor}: Following \cite{LM_22}, we denote a trunk-net:
    \begin{align}
        \mathbf{\tau}: \Rb^d\to \Rb^{P+1},\quad y=(y_1,\cdots,y_d)\mapsto \{\tau_k\}_{k=0}^P
    \end{align}
    where each $\tau_k$ is a feed-forward neural net. This can be thought of as a \textit{point-encoder} which encodes a point $y\in \Rb^d$ into a $P+1$ dimensional representation. Subsequently we denote a $\mathbf{\tau}$-induced \textit{reconstructor}:
    \begin{align}
       \Rc_{P}:= \Rc:\Rb^P\to L^2(\Omega), \quad \Rc(\alpha_k):=\tau_0(y)+\Sigma_{k=1}^{P}\alpha_k \tau_k(y).
    \end{align}
\item \textit{Projector}: Given a reconstructor $\Rc$ as above with $\tau_k\in L^2(\Omega); k=0,\cdots,P$ and $\{\tau_k\}_{k=1}^P$ as linearly independent we can define a projector:
\begin{align}
   \Pc_P:= \Pc:L^2(\Omega)\to \Rb^p,\quad \Pc(f)=\left(\langle f-\tau_0,\tau_1^*\rangle,\cdots,\langle f-\tau_0,\tau_P^*\rangle\right)
\end{align}
where $\tau_{1}^*,\cdots,\tau_P^*$ denotes the dual basis of $\tau_1,\cdots,\tau_p$ such that:
\begin{align*}
    \langle \tau_k,\tau_l^*\rangle=\delta_{kl} \quad \forall \text{ } k,l\in \{1,\cdots,P\}.
\end{align*}
    
    \end{itemize}

\begin{section}{Main Approximation Theorem}

The quality of the approximation will be determined by the extent to which the maps in Fig. \ref{Fig:a} commute. The neural operator $\Gpar$ will approximate the `true map' $\Gdex$. Recall that we denote encoding (by sensor evaluations) by $\mathcal{E}$, decoding by $\Dc$, Reconstructor map by $\Rc$, the projector map by $\Pc$, and the approximator map between finite dimensional spaces as $\mathcal{A}$. Referring to Fig. \ref{Fig:a}, notice that our neural operator will have the form $\Gpar=\Rc\circ\mathcal{A}\circ\mathcal{E}$.  
Our goal is to find an upper bound on the error incurred in approximating $\Gdex$ by $\Rc\circ\mathcal{A}\circ \mathcal{E}:=\Gpar$ over the space of shifted N-t-D operators corresponding to admissible conductivities. More particularly we will prove the following theorem.
\begin{theorem}\label{lem:2}
    Consider the sets $\Hb_{r}, D_{\Lambda}$, $\Gal$, and $\Gdex$ as above. Note $D_{\Lambda}$ is a compact set of $\Hb_{r}$. We will show that $\forall \epsilon >0$, there exist numbers $M \text{ and }P$ and continuous maps $\mathcal{E}:\Hb_{r}\to \Rb^M\times \Rb^M$, $\Rc:\Rb^P \to \tilde{\Gamma}$ and $\mathcal{A}:\Rb^{M}\times \Rb^{M}\to \Rb^{P}$ such that:
\begin{align}
\sup\limits_{\tilde{\Lambda}_{\gamma}\in D_{\Lambda}} \lvert\lvert \mathcal{G}^{\dagger}_{\text{ext}}(\tilde{\Lambda}_{\gamma})-\Rc\circ\mathcal{A}\circ\mathcal{E}(\tilde{\Lambda}_{\gamma})\rvert\rvert_{L^{2}}\leq \epsilon
    \end{align}
\end{theorem}
\begin{proof} The error term can be split in the following way,
\begin{align}\label{eq:14}
   \sup_{\tilde{\Lambda}_{\gamma}\in D_{\Lambda}}\lvert\lvert&( \Gpar-\Gdex)\tilde{\Lambda}_{\gamma}\rvert\rvert_{L^2}=\sup_{\tilde{\Lambda}_{\gamma}\in D_{\Lambda}}\lvert\lvert(\Rc\circ\mathcal{A}\circ \mathcal{E}-\Gdex )\tilde{\Lambda}_{\gamma}\rvert\rvert_{L^2}\nonumber\\
    &=\sup_{\tilde{\Lambda}_{\gamma}\in D_{\Lambda}}\lvert\lvert(\Rc\circ\mathcal{A}\circ \mathcal{E}-\Rc\circ\Pc\circ \Gdex+\Rc\circ\Pc\circ \Gdex-\Gdex)\tilde{\Lambda}_{\gamma}\rvert\rvert_{L^2}\nonumber\\
    &\leq \sup_{\tilde{\Lambda}_{\gamma}\in D_{\Lambda}}\bigg(\underbrace{\lvert\lvert(\Rc\circ\mathcal{A}\circ \mathcal{E}-\Rc\circ\Pc\circ \Gdex\circ \Dc\circ \mathcal{E})\tilde{\Lambda}_{\gamma}\rvert\rvert_{L^2}}_{E_{\mathcal{A}}}\nonumber\\
    &+\underbrace{\lvert\lvert(\Rc\circ\Pc\circ \Gdex\circ \Dc\circ \mathcal{E}-\Rc\circ\Pc\circ \Gdex)\tilde{\Lambda}_{\gamma}\rvert\rvert_{L^2}}_{E_{\mathcal{E}}}\nonumber\\&+\underbrace{\lvert\lvert(\Rc\circ\Pc\circ \Gdex-\Gdex)\tilde{\Lambda}_{\gamma}\rvert\rvert_{L^2}}_{E_{\Rc}})\bigg)
\end{align}
where $E_{\mathcal{A}}, E_{\mathcal{E}}$ and $E_{\Rc}$ denote the approximator error, encoder error, and the reconstructor error respectively.
In the next three subsections we will analyze each of these error terms separately. In particular, we will show that for any given $\epsilon$ the network parameters can be chosen to ensure that each of the error terms separately is of the order $\epsilon/3$. Here and below, the symbol $\lesssim$ means that the inequality holds upto some constant factor, and the notation $\lVert \cdot \rVert_{L^2}$ is a shorthand for $\lVert \cdot \rVert_{L^2(\Omega)}$. Furthermore, the fact that $\Omega$ is a smooth bounded domain allows us to translate upper bounds in $L^{\infty}(\Omega)$ easily into upper bounds with respect to $L^2(\Omega)$ norm, we will use this fact freely wherever needed.
\subsection{\textbf{Approximator error}}\label{sec:7.1}
First we estimate the approximation error, $E_{\mathcal{A}}$. To that end, we examine the continuity of various maps appearing in the error terms above. Recall that $\tilde{\Lambda}_{\gamma}$ is a bounded map for admissible conductivities, \cite[Theorem A.3]{hanke_11}. As the map $\Rc$ is similar to the one defined in \cite{LM_22}, the Lipschitz constant of $\Rc$ follows from \cite[Lemma 3.2]{LM_22}. Consider the encoding map $\mathcal{E}(\tilde{\Lambda}_{\gamma}):=\tilde{R}^M_{\gamma}$, then we can show that this is Lipschitz, too. This follows from the discussion in \cite[Sections 2 and 4]{reider_08} from which we first get that the map $\gamma\mapsto \tilde{R}^{M}_{\gamma}$ is Lipschitz and subsequently we apply the estimate (\cite[Theorem 2.3]{harr_19}), $ \lvert\lvert \gamma_1-\gamma_2\rvert\rvert _{L^\infty}\leq C  \lvert\lvert \tilde{\Lambda}_{\gamma_1}-\tilde{\Lambda}_{\gamma_2}\rvert\rvert_{\Lc(L^2\to L^2)}$. As $D_{\Lambda}$ is a compact set from Lemma \ref{lem:1}, then $\mathcal{E}(D_{\Lambda})$ is also compact.
Now the map, $\Gdex$ is a Lipschitz continuous map for piecewise analytic conductivities, \cite[Theorem 2.3]{harr_19}. Furthermore, the map $\Pc$ is Lipschitz by similar argument as in \cite[Lemma 3.2, 3.3]{LM_22}. $\Dc$ is also Lipschitz from the following estimate and the fact that the auxiliary projection operator,  $\tilde{P}^E_M$ is bounded, see e.g. \cite{hyv_08}.
\begin{align*}
    \lvert\lvert (\tilde{\Lambda}^M_{\gamma_1}-\tilde{\Lambda}^M_{\gamma_2})f\rvert\rvert_{L^2_{\diamond}(\partial \Omega)}\leq C(z,\phi,\Omega,d,M)\lVert \tilde{R}^M_{\gamma_1}-\tilde{R}^M_{\gamma_2}\rVert_{l^2}\lVert \tilde{P}^E_Mf\rVert_{H^s(\partial \Omega)}.
\end{align*} Thus the composition $\Pc\circ \Gdex\circ\Dc:=G: \Rb^{M\times M} \to \Rb^P$ is also Lipschitz.  Here and below, let $\mathrm{Lip(\Rc)}$ and $\mathrm{Lip(\mathcal{E})}$ denote the Lipschitz constants for the (Lipschitz) maps $\Rc$ and $\mathcal{E}$. Thus
\begin{align*}
E_{\mathcal{A}}&\leq\sup\limits_{\tilde{\Lambda}_{\gamma}\in D_{\Lambda}}\lvert\lvert(\Rc\circ\mathcal{A}\circ \mathcal{E}-\Rc\circ G\circ \mathcal{E})\tilde{\Lambda}_{\gamma}\rvert\rvert_{L^2}\\
&\leq\operatorname{Lip}(\mathcal{R}) \cdot \sup_{X \in \mathcal{E}(D_\Lambda)} \left\| \mathcal{A}(X) - G(X) \right\|,
\end{align*}
 
%One can show that $\Dc$ is Lipschitz by first using arguments from \cite[section 2, section 3]{reider_08} which shows $\gamma\mapsto \tilde{\Lambda}_{\gamma}$ is Lipschitz and then using \cite[Theorem 3.1]{harr_19}, $\lvert\lvert \gamma_1-\gamma_2\rvert\rvert _{L^\infty}\leq C  \lvert\lvert \tilde{R}^M_{\gamma_1}-\tilde{R}^M_{\gamma_2}\rvert\rvert_{\Lc(L^2\to L^2)}$.
The upper bound on the approximation error can be evaluated by finding an upper bound for the error incurred in approximating the map $G:\Rb^{M\times M}\to\Rb^P; { }X\mapsto (G_1(X),\cdots,G_P(X))$ by a neural network $\mathcal{A}$. This can be done using the results proved by Yarotsky in \cite{Yar_17}. To make this precise, we begin by approximating each of the maps $G_j:\Rb^{M\times M}\to \Rb,{ }j\in \{1,\cdots,P\}$ by a neural network $\mathcal{A}_j:\Rb^{M\times M}\to\Rb$ and then combining all the individual approximations into a single neural network $\mathcal{A}$ wherein
\begin{align*}
   \mathrm{size}(\mathcal{A})\leq P \max_{j=\{1,\cdots,P\}}\mathrm{size}(\mathcal{A}_j). 
\end{align*}
Then it follows from \cite[Theorem 1]{Yar_17} (see also discussion following \cite[Theorem 3.17]{LM_22})that for any given $\epsilon/3:=\tilde{\epsilon}>0$, there exists a neural network $\mathcal{A}_j$ with 
\begin{align*}
    \mathrm{depth}(\mathcal{A}_j)\lesssim (1+\log(\epsilon^{-1})) \quad \mathrm{size}(\mathcal{A}_j)\lesssim {{}{\tilde{\epsilon}^{-M}}}(1+\log(\tilde{\epsilon}^{-1}))
\end{align*} such that
\begin{align}
\sup\limits_{X\in \mathcal{E}(D_{\Lambda})}  \lVert G_j(X)  - \mathcal{A}_j(X)\rVert\leq \tilde\epsilon.
\end{align} Thus there exists a neural network $\mathcal{A}$ whose size is of the order  $\Oc(P\tilde{\epsilon}^{-M})$ that can approximate a Lipschitz continuous mapping $G:\Rb^{M\times M}\to\Rb^P$ within a given $\epsilon/3=\tilde{\epsilon}$ error, see also \cite[section 3.6.1]{LM_22}.

\subsection{\textbf{Encoding error}}
%Recall the continuum model and complete electrode model as considered in \cite{hyvonen_08}.
%Recall the relative N-t-D map, $\tilde{\Lambda}_{\gamma}:=\Lambda_{\gamma}-\Lambda_{1}$, where $\Lambda_{\gamma}$ is the N-t-D map when conductivity is $\gamma\in \Gamma$ and $\Lambda_{1}$ is the reference N-t-D map with conductivity 1. Similarly we defined the  relative (or, shifted) CEM operator, $\tilde{R}_{M}(\gamma)$
%The encoding error results from representing the map $\tilde{\Lambda}_{\gamma}$ by the relative CEM operator $\tilde{R}_M(\gamma)$ obtained from measurements generated in a CEM model if we allow the number of electrodes $M$ to grow very large. 
For shifted N-t-D operators $\tilde{\Lambda}_{\gamma}$ corresponding to admissible conductivities, the encoding error is given by:
\begin{align}
    \lvert\lvert(\Rc\circ\Pc\circ \Gdex\circ \Dc\circ \mathcal{E}-\Rc\circ\Pc\circ \Gdex)\tilde{\Lambda}_{\gamma}\rvert\rvert_{L^2}= \lvert\lvert\Rc\circ\Pc\circ \Gdex ((\Dc\circ \mathcal{E}-\Ib_{\Hb_r})\tilde{\Lambda}_{\gamma})\rvert\rvert_{L^2}
\end{align}
where $\Ib_{\Hb_r}$ is the identity map on $\Hb_r$. Observe that $(\Dc\circ \mathcal{E}-\Ib_{\Hb_r})\tilde{\Lambda}_{\gamma}=\tilde{\Lambda}^M_{\gamma}-\tilde{\Lambda}_{\gamma}$.
From \cite[Corollory 6.2]{garde_21}, we know that by suitably shrinking the size of electrodes as their number $M$ grows we get
\begin{lemma}\label{lem:2a}
 \cite[Theorem 7.8]{hyv_08}
  For $M\in \Nb$ denoting the number of electrodes,
  \begin{align}
  \lvert\lvert  \tilde{\Lambda}_{\gamma}- \tilde{\Lambda}^M_{\gamma}\rvert\rvert_{\Lc(L^2_{\diamond}\to L^2_{\diamond})}:=\delta_M \to 0 \quad \text{as}\quad M\to \infty.
  \end{align}
\end{lemma} 
{{}{\begin{remark}
    In this work, the CEM encoder is used because it provides a realistic model for electrode measurements and because known results show that the operator reconstructed from these measurements approximates the continuum Neumann-to-Dirichlet
operator. In particular, Lemma \ref{lem:2a} shows that as the number of electrodes increases, the
CEM based operator converges to the continuum N-t-D map. Similar
constructions are also possible for other electrode models, such as point
electrode models (see e.g. [Theorem 6.1]\cite{garde_21}), provided they admit a similar convergence property. In our analysis, the key requirement is therefore the existence of such a
convergence result for the encoder.
\end{remark}}}
As a simple consequence to the lemma above, and similar to the calculations in \cite[Lemma 5]{NiAb19} one can establish the relation between operator norm $\Lc(L^2_{\diamond}\to L^2_{\diamond})$ and the information-theoretically relevant $\Hb_r$-norm to get  $\lvert\lvert  \tilde{\Lambda}_{\gamma}- \tilde{\Lambda}^M_{\gamma}\rvert\rvert_{\Hb_r}\to 0$ as $M\to \infty$. As is shown in Lemma \ref{lem:1} below, $D_{\Lambda}$ is a compact set of $\Hb_r$. Consider the finite rank projection operators, $U_{M}:\Hb_{r}\to \Hb_{r}$ where $U_{M}:=\Dc_M\circ\mathcal{E}_M$ where $\mathcal{E}_M$ and $\Dc_M$ are as defined in section \ref{sec:5}. From Lemma \ref{lem:2a}, we have that for any $\tilde{\Lambda}_{\Gamma}\in D_{\Lambda}$:
\begin{align}
  \lim\limits_{M\to \infty} \sup\limits_{\tilde{\Lambda}_{\gamma}}\lvert\lvert {}{\tilde{\Lambda}}_{\gamma}-U_M(\tilde{\Lambda}_{\gamma})\rvert\rvert_{\Hb_r}=0. 
\end{align}
Let, $Z:=\bigcup \limits_{M\in \Nb} U_{M}(\tilde{\Lambda}_{\gamma})\cup D_{\Lambda}$. Then $Z$ is a compact set of $\Hb_{r}$ from \cite[Lemma 46]{kov_th}. Thus $\Gdex: Z\to \Gal$ is uniformly continuous on $Z$ and there exists a modulus of continuity, $\omega$ such that:
\begin{align}\label{eq:20}
    \forall z_1,z_2\in Z, \quad \lvert\lvert \Gdex(z_1)-\Gdex(z_2)\rvert\rvert_{L^{\infty}} \leq \omega(\lvert\lvert z_1-z_2\rvert\rvert_{\Hb_r})
\end{align}
On the other hand, given any $\epsilon$, we can clearly choose an $M:=M(\epsilon)$ such that
\begin{align}\label{eq:21}
    \sup\limits_{z\in D_{\Lambda}}\omega(\lvert\lvert U_{M}(z)-z\rvert\rvert_{\Hb_r})\leq \epsilon/3.
\end{align}
Finally, we know from arguments above that the maps $\Rc$ and $\Pc$ are Lipschitz continuous, see subsection \ref{sec:7.1}.
This shows that by using a large enough number of electrodes $M$, the relative CEM measurements can be used to approximate the continuum relative N-t-D operator to any arbitrary precision. More particularly, combining eqns.\eqref{eq:20} \eqref{eq:21} and the fact that $\Rc$ and $\Pc$ are Lipschitz we get that given any $\epsilon>0$, there exists $M:=M(\epsilon)$ such that,
\begin{align}
 \sup\limits_{\tilde{\Lambda}_{\gamma}\in D_{\Lambda}}\lvert\lvert(\Rc\circ\Pc\circ \Gdex\circ \Dc\circ \mathcal{E}-\Rc\circ\Pc\circ \Gdex)\tilde{\Lambda}_{\gamma}\rvert\rvert_{L^2}\lesssim  \epsilon/3.
\end{align}

\subsection{\textbf{Reconstruction error}}
Note that the output of the DeepONet lies in the space of $L^2$ functions. For any $\tilde{\Lambda}_{\gamma}\in D_{\Lambda}$, for the third term in \eqref{eq:14} we have, $E_{\Rc}\leq \sup\limits_{\tilde{\Lambda}\in D_{\Lambda}}{\lvert\lvert(\Rc\circ\Pc- \Ib_{\O})(\Gdex \tilde{\Lambda}_{\gamma})\rvert\rvert_{L^2}}$, where $\Ib_{\O}$ represents the identity map on the space of $L^2(\Omega)$ functions. Let $\tilde{\Lambda}_{\gamma}\in D_{\Lambda}$ and denote $g:=\Gdex(\tilde{\Lambda}_{\gamma})=\Gc^{\dagger}(\tilde{\Lambda}_{\gamma})=\Phi^{-1}\circ\gamma$, where $\Phi$ is the link function as described in section \ref{sec:5}. Since $\gamma$ is a piecewise analytic function, hence it belongs to some $H^s(\Omega)$ for $0<s<1/2$. It is easy to see that,
\begin{align*}
    \Rc\circ\Pc(g)=\tau_0^{\perp}+\Sigma_{k=1}^P \langle g,\tau_k^*\rangle \tau_k
\end{align*}
where $\tau_0^{\perp}:=\tau_0-\Sigma_{k=1}^P \langle \tau_0,\tau_k^*\rangle \tau_k$ is the projection of $\tau_0$ onto the orthogonal complement of $\mathrm{span}(\tau_1,\cdots,\tau_p)$. Following arguments from \cite[Lemma 3.2,3.3, 3.4]{LM_22}, we have the following Lemma.
\begin{lemma}
    For any $0<\epsilon<1$ and some $C_1$ such that $\lVert g\rVert_{H^s}\leq C_1$, there exists a trunk net $\tau:\Rb^d\to\Rb^P$ with
    \begin{align*}
       & \mathrm{size}(\mathbf{\tau})\sim 1+P\log(P/(\epsilon/3))^2\\
       &\mathrm{depth}(\mathbf{\tau})\sim 1+\log(P/(\epsilon/3))^2
    \end{align*}
such that
\begin{align*}
    E_{\Rc}\lesssim  C_1 P^{-s/d}+\epsilon/3.
\end{align*}
\end{lemma}
Here we can choose $P$ to be large enough such that, $E_{\Rc}\lesssim \epsilon/3.$

\noindent Finally, the upper bounds shown on the three kinds of error, viz. approximator error, encoder error and reconstructor error, prove Theorem \ref{lem:2}.
\end{proof}
\end{section}

\section{Network Architecture}

In this section, we describe the network architecture used in our experiments which corresponds to the one depicted in Fig. \ref{NetPic}. The input for the Branch network are the measurements of the EIT experiment. For the numerical simulations, we imagine using 16 different electrodes through which the current will be injected and the voltages measured. In practical settings, current is injected by using a pair of adjacent source and sink electrodes. For each such current pattern, voltage differences are measured across adjacent electrodes on the remaining $14$ electrodes. Hence, we obtain $13$ voltage difference measurements per injected current pattern. In total, we get $16*13 = 208$ measurements which can be arranged in the form of a matrix. {Many a times, one uses zero-padding to fill in the missing entries of the matrix to make it a square matrix.} The measurements are then fed into a series of fully connected ResNet blocks to output a vector of the same dimensionality as the output of the later defined Trunk network (See Table \ref{BranchNet}). 

\begin{table}[ht]
\centering
\caption{Branch Network Architecture}
\scriptsize
\begin{tabular}[t]{lcc }
\hline
Network Block&Input Size & Output Size\\
\hline
Two Layer ResNet& 208 & 1024 \\
Two Layer ResNet& 1024 & 1024 \\
Two Layer ResNet& 1024 & 1024 \\
Two Layer ResNet& 1024 & 512 \\
Two Layer ResNet& 512 & 512\\
Two Layer ResNet& 512 & 512\\
Two Layer ResNet& 512 & 512\\
\hline
\label{BranchNet}
\end{tabular}
\end{table}%

The Trunk network takes points $(x, y) \in \mathbb{R}^2 \subset \Omega$ as input. It then uses a series of ResNet blocks to get an output of the same dimensionality as the Branch network (See Table \ref{TrunkNet}). An inner product operator is used to combine both outputs into a prediction {{}{$\hat{\gamma}_{x, y}\in \mathbb{R}$. $\hat{\gamma}_{x, y}$} }is to be interpreted as a prediction of {{}{$\gamma$}} at point $(x, y) \in \mathbb{R}^2 \subset \Omega$ given the measurements. {{}{Each ResNet block consists of exactly two fully connected layers with a residual connection that skips these two layers.}} For all ResNet blocks in both the branch and trunk networks, we use Leaky ReLU activations. {{}{We tested three activation functions, namely: Leaky ReLU, ELU, and tanh. For our experiments, Leaky ReLU marginally outperformed the other two. } }

\begin{table}[ht]
\centering
\caption{Trunk Network Architecture}
\scriptsize
\begin{tabular}[t]{lcc }
\hline
Network Block&Input Size & Output Size\\
\hline
Two Layer ResNet& 2 & 256 \\
Two Layer ResNet& 256 & 512 \\
\hline
\label{TrunkNet} 
\end{tabular}
\end{table}%

One interesting observation from the computational perspective is that, for each EIT experiment, the Branch network only has to be computed once. The Trunk network has to be evaluated for many points (in parallel). That is why the Trunk network is chosen to be much more lightweight than the Branch network. We note that we experimented and evaluated our method with many different network architectures of similar sizes and they all performed reasonably well. {We chose to present this setup as this seemed to do the best in our experiments. }
\begin{figure}[h]
\begin{center}
\includegraphics[scale=0.5]{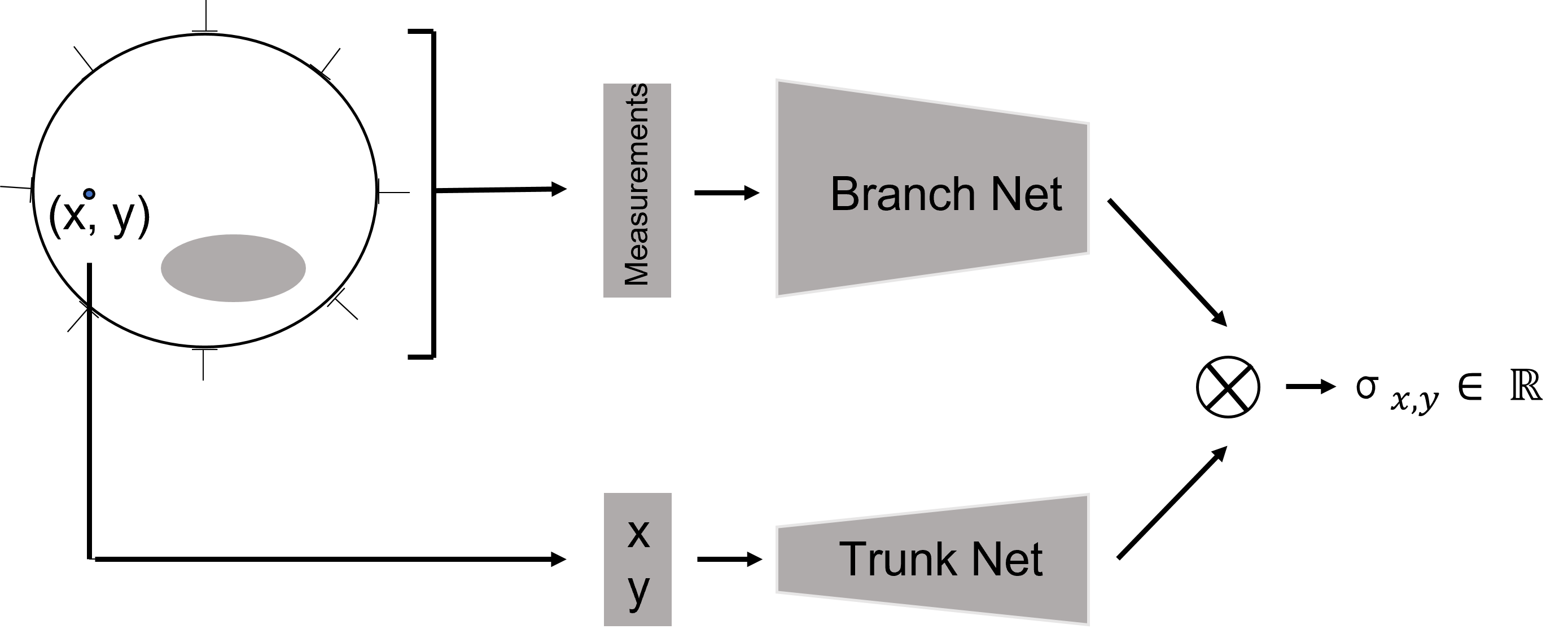}
\end{center}
\caption{DeepONet architecture used in this manuscript. It feeds the EIT measurements into a Branch Net and $2-d$ point on $\Omega$ into a Trunk Net. The outputs of both networks are combined via an inner product to predict the {{}{$\gamma_{x, y}$}} at points $(x, y)$.}\label{NetPic}
\end{figure}

\section{Experiments}

In this section, we show how our proposed Brunch-Trunk network performs in computational experiments. 

\subsection{\textbf{Training the Neural Network}}

\subsubsection{Training Data}

We used a finite element based PDE solver to generate a dataset of $16,384$ data pairs consisting of 1, 2, 3 and 4 anomalies and the corresponding measurements. That is we used $4096$ conductivity's and the corresponding measurements for each 1, 2, 3, and 4 circular anomalies. The measurements {{}{were}} obtained with a finite element based PDE solver\cite{liu2018pyeit} solving the equation \eqref{eq:1}. {{}{In detail, for each anomaly, we uniformly generated the radius $r$ in $[0.05, 0.25]$. Then we uniformly distributed the center $x, y$ of that anomaly in a circle of radius $1-r-0.01$. The reason for this is that we assume that the anomalies do not touch the boundary of the domain $\Omega$ which agrees with the standard theoretical and practical assumptions. For the cases where more than one anomaly was used, the radius and center were generated independently from each other. The conductivities of each anomaly were sampled independently and uniformly in $[5, 9]$. The training points were generated by uniformly sampling $50\%$ of the points from the anomalous and $50\%$ from the background region. }}

During training we generate a point cloud of 1024 points $(x, y) \in \mathbb{R}^2 \subset \Omega$ and the corresponding {{}{ $\gamma_{x, y}$}}. As already stated above, our sampling is to be understood as sampling $50\%$ of the points from the background and $50\%$ of the points from the anomalous regions. This is only done if there are anomalies in the conductivity profile. The reason for this kind of sampling is that in our experience it significantly increased the correct detection on small anomalies. The idea behind this is that if only a very small amount of samples is anomalous; the neural network has the tendency of predicting everything as background. {We note that more sophisticated sampling methods have been advocated in the literature. However, for the verification of our theoretical results, this simple sampling scheme tends to also work well. } 
\begin{remark}
{{}{In our experiments we used a dataset of 16,384 training pairs, which was sufficient to train the DeepONet and obtain stable reconstructions. 
We did not perform a systematic study of the minimal dataset size required. Investigating data efficiency and sample complexity is an interesting direction  for future work.}}
\end{remark}
\subsubsection{Training and Testing}

Our neural network is of the form shown in Figure \ref{NetPic}, consisting of the two components defined in Table \ref{BranchNet} and \ref{TrunkNet}. During the Training process, we first evaluate the measurement data via our Branch network followed by evaluating the corresponding large set of 1024 $2D$ points with the Trunk network. Then the inner product of both outputs is computed. Each estimated value {{}{$\hat{\sigma}_{x, y}$}} is then compared with the ground truth {{}{$\sigma_{x, y}$ }}with a weighted- $L^2$-loss. 

{{}{\begin{align*}
&L^2_{Weighted}(\gamma_{x, y}, \hat{\sigma}_{x, y}) =\\ &=\frac{\sum_{\gamma_{x, y}\in \text{background}}(\gamma_{x, y}- \hat{\gamma}_{x, y})^2}{2 \sum\mathds{1}_{background}(\gamma_{x, y})} + \frac{ \sum_{\gamma_{x, y}\in \text{anomaly}}(\gamma_{x, y}- \hat{\gamma}_{x, y})^2}{2 \sum\mathds{1}_{anomaly}(\gamma_{x, y})}
\end{align*}
Where $\mathds{1}$ refers to the indicator function to check if the conductivity at a point $\gamma_{x, y}$ belongs to the anomalies or the background. }}
%The weighted $L^2$ loss is used to reflect the fact that during training we trained our network using $50\%$ points from the anomalous area and the rest from the background area.
Furthermore, during the training process, we use a Batch size of 32 conductivity setups. The training is performed for 200 epochs using the ADAM optimizer. 
The 1-D cut plots in Figure $\ref{Images_Comparission}$ denotes the values of conductivity function along a 1-D profile indicated by a white line in the original phantom.

{{}{Note that this weighted $L^2$ loss is used only during the supervised training 
of the DeepONet. The IRGN method cannot be optimized using this loss because 
the ground truth conductivity is not available during reconstruction. 
Instead, IRGN is optimized using the standard least-squares data misfit, 
while the weighted $L^2$ metric is used only for evaluating the final reconstructions. We explain this in detail in section \ref{sec:weighted err}.}}
 \subsection{Performance} In this subsection, we describe the performance of our proposed method in our numerical studies and compare it against a standard method used for EIT inversion.
\subsubsection{Baseline: Iteratively Regularized Gauss-Newton Method}

We use the Iteratively Regularized Gauss-Newton (IRGN) method, see e.g. \cite{ahmad2019comparison}, as a baseline for our approach. The IRGN method is a popular gradient-based optimization method for solving the EIT problem in an analytical setting on a PDE mesh. We use the standard Tikhonov regularized IRGN method, see \cite{ahmad2019comparison} for a good description of this method. We do not provide an extensive comparison with other methods because our main results are on the theoretical side. Hence, the goal is to show that our method is competitive with one state-of-the-art method. 

\subsubsection{Comparison}  

We evaluate both methods on 5 images at $1\%$ additive measurement noise. The comparison results can be found in Figure \ref{Images_Comparission}. On the left-hand side is the ground truth image. The second {{}{column}} is the IRGN method followed by our proposed method. Note the first three figures are in the same color range. The rightmost image is a $1D$ cut through the previous 3 images. The cut direction is visualized in a white line in the ground truth image. We observe that both methods are approximately finding the correct location of the anomaly. However, our method outperforms the IRGN method in terms of obtaining the {{}{correct $\gamma $}} in the anomalous areas. Some of the numerical values can be found in noise study in Table \ref{noise_table}.{We also note the good generalization ability of the proposed network in practice. In particular, observe that in the third and fourth row of Fig \ref{Images_Comparission}, we attempt to reconstruct thin and stretched rectangular shaped anomalies, while the network was trained only on circular anomalies. }{}{We note that the DeepONet approach incorporates prior information from its training data, which may be beneficial when the test data has a similar setup (such as similar ranges of conductivities and shapes), but might be disadvantageous when the test data is out of the training distribution, as seen in third and fourth row of Fig \ref{Images_Comparission}, where rectangular anomalies are present. Overall, we still find DeepONet to give robust reconstruction even in out of distribution cases. On the other hand, the IRGN method does not incorporate such prior information, which may explain why the DeepONet performs better in finding the correct conductivity values in the anomalous regions. }
\begin{remark}
{{}{Besides IRGN, other regularization-based approaches such as 
total variation (TV) regularization are commonly used for EIT reconstruction. Deep learning approaches based on convolutional architectures (e.g., U-Net) have also been proposed. These methods typically operate on fixed image grids, 
whereas the DeepONet framework represents the conductivity as a continuous function.}}
\end{remark}
%\begin{table}[ht]
%\centering
%\caption{Comparison with IRGN}
%\scriptsize
%\begin{tabular}[t]{lcc }
%\hline
%&\multicolumn{2}{c}{\text{{L_2 Error at $1\%$ Noise}}}\\
%Image&IRGN & Our\\
%\hline
%Fig. \ref{Images_Comparission}, Row 1& 3.41 & 1.69 \\
%Fig. \ref{Images_Comparission}, Row 2& 1.12 & 3.85\\
%Fig. \ref{Images_Comparission}, Row 3& 2.85 & 5.12\\
%Fig. \ref{Images_Comparission}, Row 4& 1.42& 3.57\\
%Fig. \ref{Images_Comparission}, Row 5& 2.84 & 4.98\\
%\hline
%\label{Monoton_table}
%\end{tabular}
%\end{table}%

\begin{figure}[ht!]
\begin{flushleft}
 \hspace{.6cm}Ground Truth \hspace{1.0cm} IRGN Method \hspace{1.0cm} DeepONet \hspace{1.5cm} 1-D Cut\end{flushleft}
 \centering
 \includegraphics[width=0.23\textwidth]{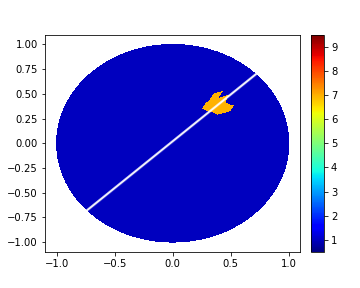}\hspace{0.01cm}
 \includegraphics[width=0.23\textwidth]{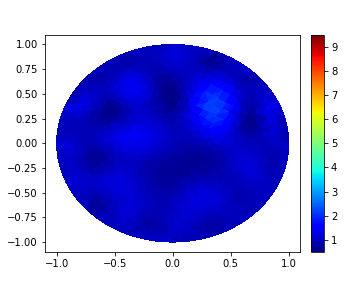}\hspace{0.01cm}
 \includegraphics[width=0.23\textwidth]{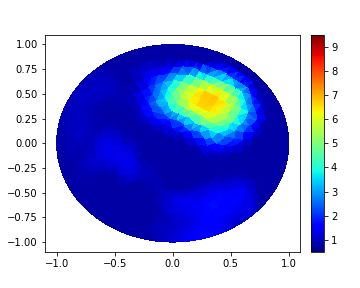}\hspace{0.01cm}
  \includegraphics[width=0.24\textwidth]{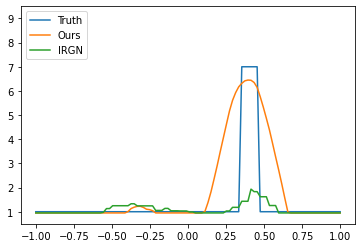}\hspace{0.01cm}\\
 \includegraphics[width=0.23\textwidth]{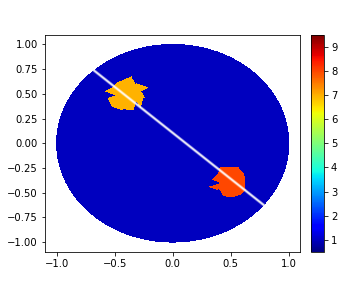}\hspace{0.01cm}
 \includegraphics[width=0.23\textwidth]{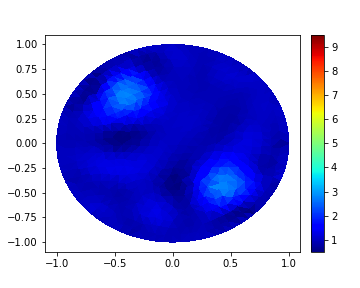}\hspace{0.01cm}
 \includegraphics[width=0.23\textwidth]{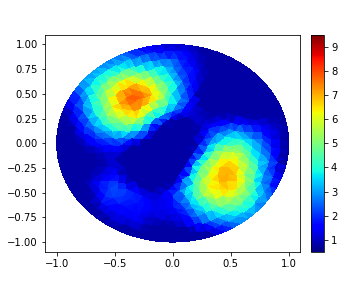}\hspace{0.01cm}
  \includegraphics[width=0.24\textwidth]{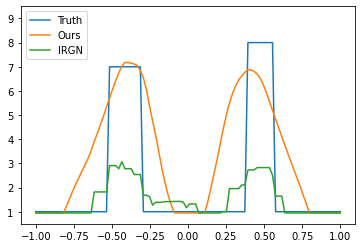}\hspace{0.01cm}\\
 \includegraphics[width=0.23\textwidth]{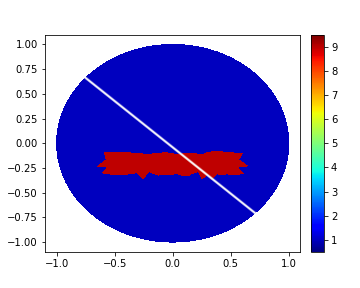}\hspace{0.01cm}
 \includegraphics[width=0.23\textwidth]{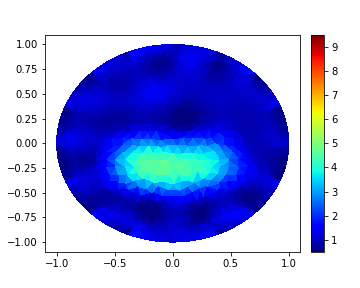}\hspace{0.01cm}
 \includegraphics[width=0.23\textwidth]{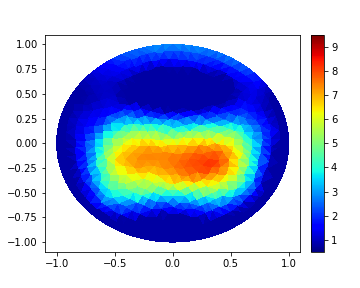}
 \hspace{0.01cm}
  \includegraphics[width=0.24\textwidth]{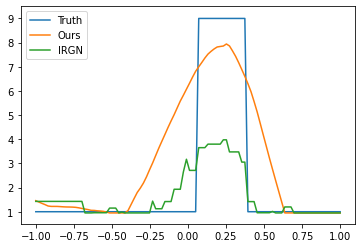}\hspace{0.01cm}\\
 \includegraphics[width=0.23\textwidth]{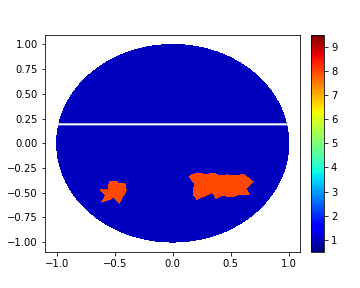}\hspace{0.01cm}
 \includegraphics[width=0.23\textwidth]{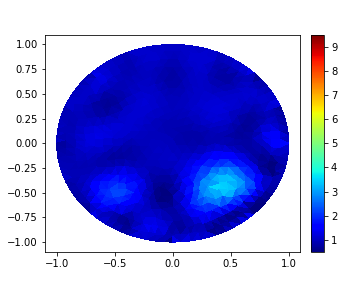}\hspace{0.01cm}
 \includegraphics[width=0.23\textwidth]{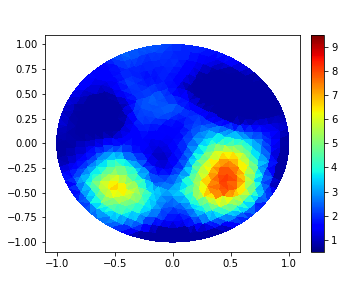}
 \hspace{0.01cm}
  \includegraphics[width=0.24\textwidth]{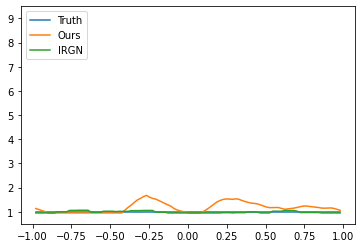}\hspace{0.01cm}\\
 \includegraphics[width=0.23\textwidth]{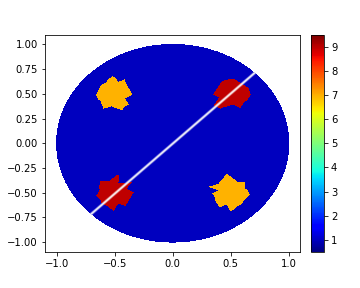}\hspace{0.01cm}
 \includegraphics[width=0.23\textwidth]{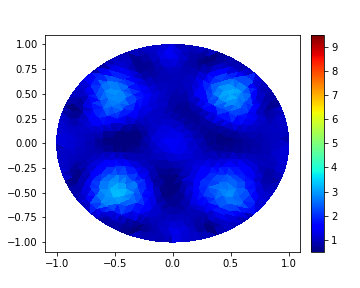}\hspace{0.01cm}
 \includegraphics[width=0.23\textwidth]{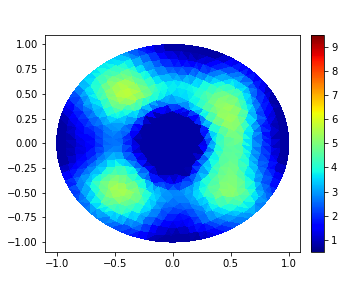}\hspace{0.01cm}
  \includegraphics[width=0.24\textwidth]{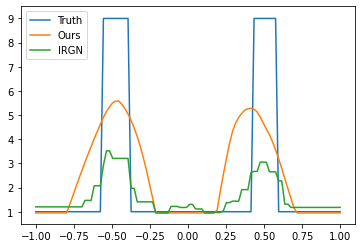}\hspace{0.01cm}

	\caption{Comparison of reconstructions of our DeepONet with the classical IRGN method. We observe that the DeepONet obtains much closer $\gamma$ values to the ground Truth in the anomalous regions. The white line in the ground truth is the cut at which the plot on the right hand side is displayed.}
	\label{Images_Comparission}
\end{figure}
\subsection{{Weighted $L^2$ error}}\label{sec:weighted err}

We believe that the standard $L^2$ error is not a useful metric here, because, for images with small anomalous regions, a small error does not imply that the location and conductivity value of the anomaly were estimated properly. For example, in such a case, the reconstruction where everything is exactly the background value, would have relatively low $L^2$ errors. However, the reconstruction would not have any practical use. Hence, we use the weighted $L^2$ error that takes $50\%$ weight of the ground truth anomalous areas and the rest from the background area. Using this criteria, we see that the DeepONet method outperforms the IRGN method as expected from seeing the images as well as the error table below.

\subsection{Noise Study}

%The stability of the method against higher noise levels on a single image is evaluated. All experiments of the errors are repeated ten times (On Fig. \ref{Images_Comparission}, Row 1) and the mean and standard deviation are displayed in Table \ref{noise_table}. Overall, the $L^2$-Error is increasing in a noise pattern in the DeepONet method. The same is true for the IRGN method. However, the IRGN method does have lower $L^2$-errors. 

%We believe that the $L^2$ error is not the correct metric because, for images with small anomalous regions, a small error does not imply that the location and conductivity value of the anomaly were estimated properly. For example, in such a case, the reconstruction where everything is exactly the background value, would have relatively low $L^2$ errors. However, the reconstruction would not have any practical use. Hence, we use a weighted $L^2$ error that takes $50\%$ weight of the ground truth anomalous areas and the rest from the background area. In this metric, the DeepONet method outperforms the IRGN method as expected from seeing the images. 

%The same patterns in both the weighted and unweighted $L^2$ errors are found in all other images that we tested. We note that for large noise levels from $7\%$, the reconstructed anomalies find significantly worse locations and shapes.  

The stability of the method against higher noise levels on a single image is evaluated. All experiments of the errors are repeated ten times (On Fig. \ref{Images_Comparission}, Row 1) and the mean and standard deviation are displayed in Table \ref{noise_table}. Overall, the weighted $L^2$-Error is increasing in a noise pattern in the DeepONet method. The same is true for the IRGN method. However, the IRGN method does have lower errors. We note that for large noise levels from $7\%$, the reconstructed anomalies find significantly worse locations and shapes. 

\begin{table}[ht]
\scriptsize
\centering
\caption{Noise Study - Weighted $L^2$ - Fig. \ref{Images_Comparission}, Row 1}
\begin{tabular}[t]{ccccccc}
\hline
Noise Level &$1\%$&$2\%$&$3\%$&$5\%$&$9\%$\\
\hline
%DeepONet - $L^2$ Error&$1.53\pm 0.14$ & $1.75 \pm 0.31$& $1.78 \pm 0.34 $&      $2.27 \pm 0.44$&$ 2.46 \pm 0.85$&$3.86 \pm 1.17$\\
%IRGN - $L^2$ Error &\boldmath{$0.37\pm 0.02$} & \boldmath{$0.42 \pm 0.03$}& \boldmath{$0.53 \pm 0.07 $}&    \boldmath{  $0.64 \pm 0.11$}&\boldmath{$ 0.81 \pm 0.20$}&\boldmath{$1.64 \pm 0.36$}\\
%\hline
DeepONet  &\boldmath{ $0.54\pm 0.52$ }& \boldmath{$0.63 \pm 0.66$}& \boldmath{$0.65 \pm 0.70 $}&  \boldmath{$ 1.05 \pm 1.25$}& \boldmath{ $1.62\pm 2.17$}\\
IRGN &$14.16\pm 0.79$ & $14.15 \pm 0.98$& $14.34 \pm 2.87 $&   $ 14.65 \pm 3.75$& $14.21 \pm 4.71$\\
\hline
\label{noise_table}
\end{tabular}
\end{table}%
\begin{remark}
While our reconstructions are already competitive, we note that recent work \cite{Lee_24} has proposed an alternative two-step training strategy for DeepONets that can further enhance recovery of sharp supports in imaging problems. Incorporating such variants into the EIT setting considered here is an exciting direction for future research, but is beyond the scope of the present work.
\end{remark}

\section{Conclusion}

In this work, we have generalized the earlier known approximation-theoretic results for approximating a map between infinite-dimensional function spaces using DeepONets to the case where we can provide similar guarantees even when the goal is to learn a map between a space of (N-t-D) operators to the space of conductivities (functions). This map arises naturally in the framework of Electrical Impedance Tomography which is a non-invasive medical imaging modality that is becoming increasingly prominent in especially several neuro-imaging applications. As deep-learning networks are becoming increasingly popular, it behooves us to go beyond black-box description of such networks and to make them  interpretable so as to facilitate deployment with confidence in critical applications such as medical imaging. Providing approximation-theoretic guarantees of the kind given in this work is a crucial first step in achieving the goal of interpretable neural-networks. In fact, we believe that the analysis carried out in this work can be easily adapted to many other problems wherein the input and output spaces are classes of operators satisfying certain general assumptions.
We would like to stress that our framework is fundamentally different from the well-studied DeepONet setting. Classical DeepONet theory deals with \emph{function-to-function} maps, where both the domain and codomain live in comparable function spaces. By contrast, in our problem the input is a Neumann-to-Dirichlet operator (an infinite-dimensional Hilbert--Schmidt object) while the output is a conductivity function in an appropriate subspace of $L^2(\Omega)$. This \emph{operator-to-function} structure requires a different treatment. In particular, existing DeepONet proofs cannot be transplanted directly; key steps such as the encoder--decoder construction, compactness arguments for the admissible operator set, and the error decomposition must be {adapted and reinterpreted} to fit this setting. Our theorem therefore extends the scope of approximation guarantees to a new class of problems beyond what has been covered in extant literature. Moreover, we have compared our approach against the widely used IRGN benchmark and show that the proposed network outperforms the baseline in numerical studies. We also performed a noise study and show that robustness of the deep-learning based against a weighted $L^2$ error metric as introduced in this work. To our knowledge, this is the first work that uses a DeepONet architecture for approximating an operator-to-function map as appearing in EIT, while also providing a theoretical guarantee to show its effectiveness. As already stated, our numerical results back-up our theoretical findings. 

\section{Appendix}
In this section, we collect some auxiliary results needed for the proof of our main theorem.
\begin{lemma}\label{lem:1}
   The set $D_{\Lambda}$ is a compact set in $\Hb_{r}$.
\end{lemma}
\begin{proof}
    First of all we note that $\Hb_{r}$ is a separable Hilbert Space with a countable orthonormal basis given in \cite [section 5.1]{NiAb19}. We recall the following theorem which is given in e.g. \cite[Chapter 2]{con90},  \cite[Theorem 197]{Rod_lec}. 
\begin{theorem}{\cite[Theorem 197]{Rod_lec}}
Let \( H \) be a separable Hilbert space, and let \( \{e_k\}_{k=1}^\infty \) be an orthonormal basis of \( H \). Then a subset \( K \subset H \) is compact iff:
\begin{enumerate}
    \item \( K \) is closed,
    \item \( K \) is bounded, and
    \item \( K \) has {equi-small tails} with respect to \( \{e_k\} \), i.e., for every \( \varepsilon > 0 \), there exists \( N \in \mathbb{N} \) such that for all \( \vec{v} \in K \),
    \[
    \sum_{k=N}^\infty |\langle \vec{v}, e_k \rangle|^2 < \varepsilon^2.
    \]
\end{enumerate}
\end{theorem} Now, the fact that $D_{\Lambda}$ is bounded follows from \cite[(2.5)]{garde_21} as before. The fact that $D_{\Lambda}$ has equi-small tails follows from \cite[Lemma 4]{NiAb19} in which we choose $J,K$ and $\nu$ such that $C\min(J,K)^{-\nu}\leq \varepsilon$ for any given $\varepsilon$. Thus it remains to show that $D_{\Lambda}$ is closed. To that end, choose a sequence of relative N-t-D operators $(\tilde{\Lambda}_{\gamma_n}: \gamma_n\in \Gamma)$ such that they converge to some operator $T$ in $\Hb_r$ i.e., $\tilde{\Lambda}_{\gamma_n}\to T$. We will show that there exists some $\gamma\in \Gamma$ such that $T=\tilde{\Lambda}_{\gamma}$ i.e. $T$ is also a N-t-D operator. First of all note that since $(\tilde{\Lambda}_{\gamma_{n}})$ converges, hence  $(\tilde{\Lambda}_{\gamma_{n}})$ is actually Cauchy in the $\Hb_r$ norm. Using the relation between the operator norm and the information-theoretic Hilbert-Schmidt norm as stated in \cite [Lemma 5]{NiAb19}, we can show that $(\Lambda_{\gamma_{n}})$ is also Cauchy in the operator norm $\Lc(L^2\to L^2)$. By using the inverse continuity result for the Calderon problem, see \cite [Theorem 2.3]{harr_19}, we get that the sequence $\gamma_{n}$ is Cauchy with respect to the $L^{\infty}$ norm and hence converges, say $\gamma_n\to \gamma_0 \in L^{\infty}(\Omega)\cap \Gamma$. Let us now consider, $\Lambda_{\gamma_0}$. From forward continuity result for the Calderon problem, see \cite[Lemma 2.5]{harr_19}, $\Lambda_{\gamma_n}\to \Lambda_{\gamma_0}$. Now from uniqueness of the limit, $\Lambda_{\gamma_0}=T$. Hence the set, $D_{\Lambda}$ satisfies all three criteria for it be a compact subset of $\Hb_{r}$.
\end{proof}

\begin{lemma}\cite[Theorem 4.1]{Dug_51}\label{th:dug}
  Let $X$ be an arbitrary metric space, $A$ a closed subset of $X$, $L$ a locally convex linear space, and $f:A\to L$ be a continuous map. Then there exists an extension $F:X\to L$ (which is continuous by definition) such that $F(a)=f(a)$ for all $a\in A$.   
\end{lemma}

Finally, consider the encoding map $\mathcal{E}(\tilde{\Lambda}):=\mathcal{E}_M(\tilde{\Lambda})=\tilde{R}^M_{\gamma}.$  We have the following continuity results:
\begin{lemma}\cite[section 2]{reider_08}\label{lem:7.3}
The map $T:\gamma\to \tilde{R}_{\gamma}^M$ is Lipschitz continuous, i.e. $\lvert\lvert \tilde{R}^M_{\gamma_1}-\tilde{R}^M_{\gamma_2}\rvert\rvert_{\Lc(L^2\to L^2)}\leq \lvert\lvert \gamma_1-\gamma_2\rvert\rvert _{L^\infty}$
\end{lemma}
\begin{lemma}\cite[Theorem 2.3]{harr_19}\label{lem:7.4}
$\lvert\lvert \gamma_1-\gamma_2\rvert\rvert _{L^2}\leq c \lvert\lvert \gamma_1-\gamma_2\rvert\rvert _{L^\infty}\leq C  \lvert\lvert \tilde{\Lambda}_{\gamma_1}-\tilde{\Lambda}_{\gamma_2}\rvert\rvert_{\Lc(L^2\to L^2)}$
\end{lemma}
From previous two Lemmas, we get that the map $\mathcal{E}$ is Lipschitz continuous.

%\begin{acknowledgements}
%\end{acknowledgements}

 \section*{There is no conflict of interest.}
 \section*{The relevant codes used for the experimental section in the article can be made available by the authors on reasonable request.}

\bibliographystyle{abbrv}
\bibliography{refs}

\end{document}